\documentclass{article}

\usepackage[final,nonatbib]{neurips_2020}

\usepackage[natbib=true,
style=ieee,
citestyle=numeric-comp,
maxcitenames=1, mincitenames=1,
maxbibnames=999,
url=false,isbn=false,doi=false]{biblatex}
\DefineBibliographyStrings{english}{
  andothers = {et al\adddot}
}
\AtEveryBibitem{
  \clearlist{language}
  \clearfield{series}
  \ifentrytype{book}{}{
      \clearfield{issn} 
      \clearfield{doi} 
      \clearlist{address}
      \clearfield{address}
      \clearfield{month}
      \clearname{editor}
      \clearfield{eprint}
      \clearlist{location}
      \clearfield{chapter}
  }
  \ifentrytype{online}{}{
    \clearfield{url}
  }
  \ifentrytype{inproceedings}{
    \clearlist{publisher}
  }{}
}
\usepackage[utf8]{inputenc} 
\usepackage[T1]{fontenc}    
\usepackage{hyperref}       
\usepackage{url}            
\usepackage{booktabs}       
\usepackage{amsfonts}       
\usepackage{nicefrac}       
\usepackage{microtype}      

\newcommand{\ArchCondText}{Architecture condition}

\usepackage{amsmath,amsfonts,bm}
\usepackage{amsthm}
\newtheorem{theorem}{Theorem}

\newtheorem{lemma}{Lemma}
\newtheorem{fact}{Fact}

\theoremstyle{definition}
\newtheorem{definition}{Definition}
\theoremstyle{remark}

\usepackage{tikz}
\usetikzlibrary{shapes.geometric}
\usetikzlibrary {shapes.misc}
\usetikzlibrary{positioning}
\usepackage{lscape}
\usepackage{fontawesome}

\usepackage{subcaption}

\usepackage{enumerate}

\usepackage{mleftright}

\usepackage[toc,page]{appendix}

\newcommand{\titleText}{Universal Approximation Property of \\ Neural Ordinary Differential Equations}
\newcommand{\titleSentenceCase}{Universal approximation property of neural ordinary differential equations.}

\title{\titleText{}}

\author{
    Takeshi Teshima\\
    The University of Tokyo, RIKEN\\
    \texttt{teshima@ms.k.u-tokyo.ac.jp}\\
    \And
    Koichi Tojo\\
    RIKEN\\
    \texttt{koichi.tojo@riken.jp}\\
    \And
    Masahiro Ikeda\\
    RIKEN\\
    \texttt{masahiro.ikeda@riken.jp}\\
    \And
    Isao Ishikawa\\
    Ehime University, RIKEN\\
    \texttt{ishikawa.isao.zx@ehime-u.ac.jp}\\
    \And
    Kenta Oono\\
    The University of Tokyo\\
    \texttt{kenta\_oono@mist.i.u-tokyo.ac.jp}\\
}

\newcommand{\acknowledgmentContent}{
The authors would like to thank the anonymous reviewers for the insightful discussions.
This work was supported by RIKEN Junior Research Associate Program.
TT was supported by Masason Foundation.
II and MI were supported by CREST:JPMJCR1913.
}

\hypersetup{
pdftitle={Universal Approximation Property of Neural Ordinary Differential Equations},
pdfsubject={},
pdfauthor={Takeshi Teshima, Koichi Tojo, Masahiro Ikeda, Isao Ishikawa, Kenta Oono},
pdfkeywords={Invertible neural network, Neural ordinary differential equations, Diffeomorphism, Universal approximation property, Universality},
} \usepackage{xcolor}

\newcommand{\status}[1]{} \newcommand{\Restrict}[2]{#1\vert_{#2}}

\def \R {\mathbb{R}}
\def \Re {\mathbb{R}}
\newcommand{\Na}{\mathbb{N}}

\newcommand{\Identity}{\mathrm{Id}}

\newcommand{\x}{\mbox{\boldmath $x$}}
\newcommand{\y}{\mbox{\boldmath $y$}}

\newcommand{\LpKnorm}[1]{\left\Vert #1 \right\Vert_{p, K}}
\newcommand{\supRangenorm}[2]{\left\Vert #2 \right\Vert_{\sup, #1}}

\newcommand{\inftyKnorm}[1]{\supRangenorm{K}{#1}}
\newcommand{\supKnorm}[1]{\supRangenorm{K}{#1}}
\newcommand{\Euclideannorm}[1]{\left\Vert #1 \right\Vert}
\newcommand{\opnorm}[1]{\left\Vert #1 \right\Vert_{\mathrm{op}}}

\newcommand{\Dtwo}{\CtwoDomainDiff}
\newcommand{\DcRDCmd}[1]{\mathrm{Diff}^{#1}_\mathrm{c}}
\newcommand{\DcRD}{\DcRDCmd{2}}
\newcommand{\DcRDr}{\DcRDCmd{r}}
\newcommand{\CtwoDomainDiff}{{\mathcal{D}^2}}

\newcommand{\FlowEnds}[1]{S^{#1}}

\newcommand{\ReD}{\Re^d}

\newcommand{\Affine}{\mathrm{Aff}}

\newcommand{\Lipsp}{\operatorname{Lip}(\R^d)}
\newcommand{\LipConst}[1]{L_{#1}}

\newcommand{\supp}[1]{\mathrm{supp}\ #1} \newcommand{\INN}[1]{\mathrm{INN}_{#1}}

\newcommand{\ARFINNModel}{\mathcal{M}}

\newcommand{\INNModelGeneric}{\mathcal{M}} \newcommand{\FLin}{\mathrm{Aff}}

\newcommand{\FNODE}[1]{#1\text{-}\mathrm{NODE}}

\newcommand{\NODEJacobiFuncClass}{\mathcal{H}}

\newcommand{\HFNODE}{\FNODE{\NODEJacobiFuncClass}}

\newcommand{\INNHNODE}{\INN{\HFNODE}}

\newcommand{\IVPFunc}[1]{\mathrm{IVP}[#1]}
\newcommand{\IVP}[3]{\IVPFunc{#1}(#2, #3)}

\newcommand{\ODEFlowEnds}[1]{\Psi(#1)} \newcommand{\INNs}{INNs}
\newcommand{\NODEs}{NODEs}

\newcommand{\Gronwall}{Gr\"{o}nwall}
 \newcommand{\targetError}[1]{\Delta_{\x_0}(#1)}
\newcommand{\BoundDef}{\delta e^{\LipConst{F}}}
\newcommand{\Bound}{B}  
\begin{document}
\maketitle
\begin{abstract}
Neural ordinary differential equations (NODEs) is an invertible neural network architecture promising for its free-form Jacobian and the availability of a tractable Jacobian determinant estimator.
Recently, the representation power of NODEs has been partly uncovered: they form an $L^p$-universal approximator for continuous maps under certain conditions.
However, the $L^p$-universality may fail to guarantee an approximation for the entire input domain as it may still hold even if the approximator largely differs from the target function on a small region of the input space.
To further uncover the potential of NODEs, we show their stronger approximation property, namely the \emph{$\sup$-universality} for approximating a large class of diffeomorphisms. It is shown by leveraging a structure theorem of the diffeomorphism group, and the result complements the existing literature by establishing a fairly large set of mappings that NODEs can approximate with a stronger guarantee.
\end{abstract}

\section{Introduction}
\emph{Neural ordinary differential equations} (NODEs) \cite{ChenNeural2018a} are a family of deep neural networks that indirectly model functions by transforming an input vector through an ordinary differential equation (ODE).
When viewed as an invertible neural network (INN) architecture, NODEs have the advantage of having free-form Jacobian, i.e., it is invertible without restricting the Jacobian's structure, unlike other INN architectures \cite{PapamakariosNormalizing2019}.
For the out-of-box invertibility and the availability of a tractable unbiased estimator of the Jacobian determinant \cite{GrathwohlFFJORD2018}, NODEs have been used for constructing \emph{continuous normalizing flows} for generative modeling and density estimation \cite{ChenNeural2018a,GrathwohlFFJORD2018,FinlayHow2020}.

Recently, the representation power of NODEs has been partly uncovered in \citet{LiDeep2020}, namely, a sufficient condition for a family of NODEs to be an \emph{\(L^p\)-universal approximator} (see Definition~\ref{def: sup universality}) for continuous maps has been established.
However, the universal approximation property with respect to the $L^p$-norm can be insufficient as it does not guarantee an approximation for the entire input domain: $L^p$ approximation may still hold even if the approximator largely differs from the target function on a small region of the input space.

In this work, we elucidate that the NODEs are a $\sup$-universal approximator (Definition~\ref{def: sup universality}) for a fairly large class of \emph{diffeomorphisms}, i.e., smooth invertible maps with smooth inverse.
Our result establishes a function class that can be approximated using NODEs with a stronger guarantee than in the existing literature \cite{LiDeep2020}.
We prove the result by using a structure theorem of \emph{differential geometry} to represent a diffeomorphism as a finite composition of \emph{flow endpoints}, i.e., diffeomorphisms that are smooth transformations of the identity map.
The NODEs are themselves examples of flow endpoints, and we derive the main result by approximating the flow endpoints by the NODEs. \section{Preliminaries and goal}
In this section, we define the family of NODEs considered in the present paper as well as the notion of universality.

\subsection{Neural ordinary differential equations (NODEs)}
Let $\Re$ (resp. $\Na$) denote the set of all real values (resp. all positive integers).
Throughout the paper, we fix $d \in \Na$.
Let $\Lipsp{}:=\{ f\colon \R^d\to \R^d\ |\ f \text{ is Lipschitz continuous} \}$.
It is known that any \emph{autonomous} ODE (i.e., one that is defined by a time-invariant vector field) with a Lipschitz continuous vector field has a solution and that the solution is unique:
\begin{fact}[Existence and uniqueness of a global solution to an ODE \cite{Derrickglobal1976}]\label{fact:ODE solution exists for Lip}
Let $f \in \Lipsp{}$.
Then, a solution $z\colon \Re\to \R^d$ to the following ordinary differential equation exists and it is unique:
\begin{equation}\label{eq:initial value problem}
z(0) = \x, \quad \dot{z}(t)= f(z(t)), \quad t \in \Re,
\end{equation}
where $\x\in \R^d$, and $\dot{z}$ denotes the derivative of $z$.
\end{fact}

In view of Fact~\ref{fact:ODE solution exists for Lip}, we use the following notation.
\begin{definition}
\label{def:ivp}
For $f \in \Lipsp{}$, $\x \in \R^d$, and $t \in \Re$, we define
\[
\IVP{f}{\x}{t} := z(t),
\]
where $z: \Re \to \R^d$ is the unique solution to Equation~\eqref{eq:initial value problem}.
\end{definition}

\begin{definition}[Autonomous-ODE flow endpoints; \citet{LiDeep2020}]\label{def: autonomous ODE flow endpoints}
For $\mathcal{F}\subset \Lipsp{}$, we define
\[
\ODEFlowEnds{\mathcal{F}}:= \{\IVP{f}{\cdot}{1} \ |\ f\in \mathcal{F}\}.
\]
\end{definition}

\begin{definition}[\(\INNHNODE\)]
Let $\Affine$ denote the group of all invertible affine maps on $\R^d$,
and let \(\NODEJacobiFuncClass \subset \Lipsp{}\). Define the invertible neural network architecture based on NODEs as
\[
\INNHNODE := \{W \circ \psi_{k} \circ \cdots \circ \psi_{1} \ |\ \psi_1, \ldots, \psi_k \in \ODEFlowEnds{\NODEJacobiFuncClass{}}, W \in \Affine{}, k \in \Na\}.
\]
\end{definition}

\subsection{Goal: the notions of universality and their relations}
Here, we define the notions of universality.
Let \(m, n \in \Na\).
For a subset $K\subset\Re^m$ and a map $f: K \to \Re^n$, we define $\inftyKnorm{f}:=\sup_{x\in K}\Euclideannorm{f(x)}$, where $\|\cdot\|$ denotes the Euclidean norm.
Also, for a measurable map $f:\Re^m\to\Re^n$, a subset $K\subset\Re^m$, and $p\in [1, \infty)$, we define \(\LpKnorm{f} := \left(\int_K \Euclideannorm{f(x)}^p dx\right)^{1/p}\).
\begin{definition}[\(\sup\)-universality and \(L^p\)-universality]
\label{def: sup universality}
Let $\INNModelGeneric$ be a model, which is a set of measurable mappings from $\Re^m$ to $\Re^n$.
Let $\mathcal{F}$ be a set of measurable mappings $f:U_f\rightarrow\Re^n$, where $U_f$ is a measurable subset of $\Re^m$, which may depend on $f$.
We say that $\INNModelGeneric$ is a \emph{$\sup$-universal approximator} or \emph{has the $\sup$-universal approximation property} for $\mathcal{F}$ if for any $f\in \mathcal{F}$, any $\varepsilon>0$, and any compact subset $K\subset U_f$, there exists $g\in \INNModelGeneric$ such that $\supKnorm{f - g}<\varepsilon$.
The \(L^p\)-universal approximation property is defined by replacing \(\supKnorm{\cdot}\) with \(\LpKnorm{\cdot}\) in the above.
\end{definition}

\paragraph{Our goal.}
Our goal is to elucidate the representation power of \INNs{} composed of \NODEs{} by proving the \(\sup\)-universality of $\INNHNODE$ for a fairly large class of \emph{diffeomorphisms}, i.e., smooth invertible functions with smooth inverse.

 \section{Main result}
In this section, we present our main result, Theorem~\ref{thm: NODE is sup-universal}.

First, we define the following class of invertible maps, which will be our target to be approximated.
\begin{definition}[\(C^2\)-diffeomorphisms: $\CtwoDomainDiff$]
\label{def: D}
We define $\CtwoDomainDiff$ as the set of all $C^2$-diffeomorphisms $f:U_f\rightarrow {\rm Im}(f)\subset\ReD$ 
, where $U_f \subset \ReD$ is open and $C^2$-diffeomorphic to $\ReD$, and it may depend on $f$.
\end{definition}
The set $\CtwoDomainDiff$ is a fairly large class: it contains any \(C^2\)-diffeomorphism defined on the entire \(\ReD\), an open convex set, or more generally, a star-shaped open set.

Now, we state our main result to establish a class that the invertible neural networks based on NODEs can approximate with respect to the \(\sup\)-norm.
\begin{theorem}[Universality of NODEs]\label{thm: NODE is sup-universal}
Assume \(\NODEJacobiFuncClass \subset \Lipsp{}\) is a $\sup$-universal approximator for $\Lipsp{}$.
Then, \(\INNHNODE\) is a \(\sup\)-universal approximator for \(\Dtwo\).
\end{theorem}
Examples of $\mathcal{H}$ include the multi-layer perceptron with finite weights and Lipschitz-continuous activation functions such as rectified linear unit (ReLU) activation \cite{LeCunDeep2015,ChenNeural2018a}, as well as the \emph{Lipschitz Networks} \citep[Theorem~3]{AnilSorting2019}.

\paragraph{Proof outline.}
To prove Theorem~\ref{thm: NODE is sup-universal}, we take a similar strategy to that of Theorem~1 of~\cite{TeshimaCouplingbased2020} but with a major modification to adapt to our problem.
First, the approximation target is reduced from \(\CtwoDomainDiff\) to the set of compactly-supported diffeomorphisms from $\ReD$ to $\ReD$, denoted by \(\DcRD\), by applying Fact~\ref{red to comp. supp. diff} in Appendix~\ref{sec:appendix:universality-proof-prep}.
Then, it is shown that we can represent each $f\in\DcRD$ as a finite composition of \emph{flow endpoints} (Definition~\ref{def: appendix flow endpoints} in Appendix~\ref{sec:appendix:universality-proof-prep}), each of which can be approximated by a NODE.
The decomposition of $f$ into flow endpoints is realized by relying on a structure theorem of $\DcRD$ (Fact~\ref{fact: simplicity} in Appendix~\ref{sec:appendix:universality-proof-prep}) attributed to Herman, Thurston \cite{ThurstonFoliations1974}, Epstein \cite{Epsteinsimplicity1970}, and Mather \cite{MatherCommutators1974, MatherCommutators1975}.
Note that we require a different definition of flow endpoints (Definition~\ref{def: appendix flow endpoints} in Appendix~\ref{sec:appendix:universality-proof-prep}) from that employed in~\citep[Corollary~2]{TeshimaCouplingbased2020} in order to incorporate sufficient smoothness of the underlying flows.

 \section{Related work and Discussion}
In this section, we overview the existing literature on the representation power of NODEs to provide the context of the present paper.

\paragraph{\(L^p\)-universal approximation property of NODEs.}
\citet{LiDeep2020} considered NODEs capped with a \emph{terminal family} to map the output of NODEs to a vector of the desired output dimension, and its Proposition~3.8 showed that the model class has the \(L^p\)-universality for the set of all continuous maps from \(\ReD\) to \(\Re^n\) (\(n \in \Na\)), under a certain sufficient condition.
In comparison to our result here, the result of \citet{LiDeep2020} established the universality of NODEs for a larger target function class (namely continuous maps) with a weaker notion of approximation (namely \(L^p\)-universality).

\paragraph{Limitations on the representation power of NODEs.}
\citet{ZhangApproximation2020} formulated its Theorem~1 to show that NODEs are not universal approximators by presenting a function that a NODE cannot approximate.
The existence of this counterexample does not contradict our result because our approximation target \(\CtwoDomainDiff\) is different from the function class considered in \citet{ZhangApproximation2020}: the class in \citet{ZhangApproximation2020} can contain discontinuous maps whereas the elements of \(\CtwoDomainDiff\) are smooth and invertible.

\paragraph{Universality of augmented NODEs.}
As a device to enhance the representation power of NODEs, increasing the dimensionality and padding zeros to the inputs/outputs has been explored \citep{DupontAugmented2019a,ZhangApproximation2020}.
\citet{ZhangApproximation2020} showed that the augmented NODEs (ANODEs) are universal approximators for homeomorphisms.
The approach has a limitation that it can undermine the invertibility of the model: unless the model is ideally trained so that it always outputs zeros in the zero-padded dimensions, the model can no longer represent an invertible map operating on the original dimensionality.
On the other hand, the present work explores the universal approximation property of NODEs that is achieved without introducing the complication arising from the dimensionality augmentation.

\paragraph{Relation between \(\INNHNODE\) and time-dependent NODEs.}
Our result can be readily extended to the design choice of NODEs that includes the time-index as an argument of $f$. It can be done by limiting our attention to the subset of the considered class of $f$ consisting of all time-invariant ones as in the following.
Let \(a \in (0, \infty]\) and consider \(\tilde f : \ReD\times(-a, a)\) be such that there exists a continuous function \(\ell: (-a, a) \to \Re_{\geq 0}\) satisfying
\begin{align*}
\|\tilde f(\x_1, t) - \tilde f(\x_2, t)\| \leq \ell(t)\|\x_1 - \x_2\|.
\end{align*}
Then, the initial value problem
\[z(0) = \x, \quad \dot{z}(t)= \tilde f(z(t), t), \quad t \in (-a, a)\]
has a solution $z: (-a, a) \to \ReD$ and it is unique \cite{Derrickglobal1976}, synonymously to Fact~\ref{fact:ODE solution exists for Lip}.
Then, given a set \(\tilde{\NODEJacobiFuncClass{}}\) of such mappings $\tilde f$, we can consider its subset \(\NODEJacobiFuncClass{}\) that contains only the time-invariant elements, i.e., \(\NODEJacobiFuncClass{} \subset \tilde{\NODEJacobiFuncClass{}}\) such that for any \(f \in \NODEJacobiFuncClass{}\) and any \(\x \in \ReD\), \(f(\x, \cdot)\) is a constant mapping. Such an \(f\) is an element of \(\Lipsp{}\) with \(\inf_{t \in (-a, a)} \ell(t) \geq 0\) being a Lipschitz constant.
Then, we can apply Theorem~\ref{thm: NODE is sup-universal} to \(\NODEJacobiFuncClass{}\) and its induced \(\INNHNODE{}\).

\section{Conclusion}
In this paper, we uncovered the $\sup$-universality of the INNs composed of NODEs for approximating a large class of diffeomorphisms.
This result complements the existing literature that showed the weaker approximation property of NODEs, namely $L^p$-universality, for general continuous maps.
Whether the $\sup$-universality holds for a larger class of maps than \(\CtwoDomainDiff\) is an important research question for future work.
Also, it is important for future work to quantitatively evaluate how many layers of NODEs are required to approximate a given diffeomorphism with a specified smoothness such as a bi-Lipschitz constant to evaluate the efficiency of the approximation. \begin{ack}
\acknowledgmentContent{}
\end{ack}
\printbibliography
 \clearpage
\begin{appendices}
\global\csname @topnum\endcsname 0
\global\csname @botnum\endcsname 0

This is the Supplementary~Material for ``\titleSentenceCase{}''
Table~\ref{tbl:notation-table} summarizes the abbreviations and the symbols used in the paper.

\begin{table}[tbph]
  \caption{Abbreviation and notation table.}
  \label{tbl:notation-table}
  \centering
  \begin{tabular}{ll}
    \toprule
    Abbreviation/Notation & Meaning \\
    \midrule
    INN & Invertible neural networks\\
    NODE & Neural ordinary differential equations\\
    \midrule
    $\FLin$ & Set of invertible affine transformations\\
    \(\IVP{f}{\x}{t}\) & The (unique) solution to an initial value problem evaluated at $t$ \\
    \(\ODEFlowEnds{\mathcal{F}}\) & Set of NODEs obtained from the Lipschitz continuous vector fields $\mathcal{F}$ \\
    \(\Lipsp\) & The set of all Lipschitz continuous maps from $\ReD$ to $\ReD$ \\
    \(\INNHNODE{}\) & INNs composed of $\FLin$ and NODEs parametrized by \(\NODEJacobiFuncClass \subset \Lipsp{}\)\\
    \midrule
    $d \in \Na$ & Dimensionality of the Euclidean space under consideration\\
    $\CtwoDomainDiff$ & Set of all $C^2$-diffeomorphisms with \(C^2\)-diffeomorphic domains\\
    $\DcRDr$ & Group of compactly-supported $C^r$-diffeomorphisms on \(\ReD\) ($1 \leq r \leq \infty$)\\
    \midrule
    $\Euclideannorm{\cdot}$ & Euclidean norm\\
    $\opnorm{\cdot}$ & Operator norm\\
    $\inftyKnorm{\cdot}$ & Supremum norm on a subset $K\subset \mathbb{R}^d$\\
    $\LpKnorm{\cdot}$ & \(L^p\)-norm on a subset $K\subset \mathbb{R}^d$\\
    \(\Identity\) & Identity map \\
    \(\supp{}\) & Support of a map \\
    \bottomrule
  \end{tabular}
\end{table}

\section{Proof of Theorem~\ref{thm: NODE is sup-universal}}
\label{sec:appendix:universality-proof}
Here, we provide a proof of Theorem~\ref{thm: NODE is sup-universal}.
In Section~\ref{sec:appendix:universality-proof-prep}, we display the known facts and show the lemmas used for the proof.
In Section~\ref{sec:appendix:universality-proof-content}, we prove Theorem~\ref{thm: NODE is sup-universal}.

\subsection{Lemmas and known facts}
\label{sec:appendix:universality-proof-prep}
We use the following definition and facts from \citet{TeshimaCouplingbased2020}.
\begin{definition}[Compactly supported diffeomorphism]
We use $\DcRDr$ to denote the set of all compactly supported \(C^r\)-diffeomorphisms ($1 \leq r \leq \infty$) from $\R^d$ to $\R^d$.
Here, we say a diffeomorphism $f$ on $\mathbb{R}^d$ is {\em compactly supported} if 
there exists a compact subset $K\subset \mathbb{R}^d$ such that for any $x\notin K$, $f(x)=x$.
We regard $\DcRDr$ as a group whose group operation is function composition.
\end{definition}

The following fact enables us to reduce the approximation problem for $\CtwoDomainDiff$ to that for $\DcRD$.
\begin{fact}[Lemma~5 of \citet{TeshimaCouplingbased2020}]
\label{red to comp. supp. diff}
Let $f \colon U \to \ReD$ be an element of $\CtwoDomainDiff$, and let $K\subset U$ be a compact set. Then, there exists $h \in \DcRD$ and an affine transform $W \in \FLin$ such that \[\Restrict{W\circ h}{K}=\Restrict{f}{K}.\]
\end{fact}

The following fact enables the component-wise approximation, i.e., given a transformation that is represented by a composition of some transformations, we can approximate it by approximating each constituent and composing them.

\begin{fact}[Compatibility of composition and approximation;
Proposition~6 of \citet{TeshimaCouplingbased2020}]\label{lemma:composition}
Let \(\ARFINNModel\) be a set of locally bounded maps from $\ReD$ to $\ReD$,
and $F_1,\dots,F_k$ be continuous maps from $\ReD$ to $\ReD$.
Assume for any $\varepsilon>0$ and any compact set $K\subset\mathbb{R}^d$, there exist $\widetilde{G}_1,\dots, \widetilde{G}_k\in \ARFINNModel$ such that, for $1 \leq i \leq k$, $\big\Vert F_i-\widetilde{G}_i\big\Vert_{\sup, K}<\varepsilon$.
Then for any $\varepsilon>0$ and any compact set $K\subset\mathbb{R}^d$, there exist $G_1, \dots, G_k\in \ARFINNModel$ such that
\[\supKnorm{F_k\circ\cdots \circ F_1-G_k\circ\cdots\circ G_1} < \varepsilon.\]
\end{fact}

The following fact is attributed to Herman, Thurston \cite{ThurstonFoliations1974}, Epstein \cite{Epsteinsimplicity1970}, and Mather \cite{MatherCommutators1974, MatherCommutators1975}. See Fact~2 of \citet{TeshimaCouplingbased2020} and the remarks therein for details.
Let $\Identity{}$ denote the identity map.
\begin{fact}[Fact~2 of \citet{TeshimaCouplingbased2020}]
\label{fact: simplicity}
If \(r \neq d + 1\), the group $\DcRDr$ is simple, i.e., any normal subgroup $H \subset \DcRDr$ is either $\{\Identity\}$ or $\DcRDr$.
\end{fact}

Next, we define a subset of $\DcRDr$ called the \emph{flow endpoints}.
In Lemma~\ref{lem: diffc2 is generated by flow endpoints}, it is shown that the set of flow endpoints generates a non-trivial normal subgroup of $\DcRDr$.
Therefore, by combining it with Fact~\ref{lemma:composition}, we can represent any element of $\DcRDr$ as a finite composition of flow endpoints, each of which can be approximated by the NODEs.

While Corollary~2 of \citet{TeshimaCouplingbased2020} also defined a set of flow endpoints in $\DcRD$, it differs from the one defined here which is tailored for our purpose.
The two definitions can be interpreted as describing two different generators of the same group $\DcRD$.
Let $\mathrm{supp}$ denote the support of a map.
\begin{definition}[Flow endpoints $\FlowEnds{r}$ in $\DcRDr$]
\label{def: appendix flow endpoints}
Let $1 \leq r \leq \infty$. Let $\FlowEnds{r}\subset\DcRDr$ be the set of diffeomorphisms $g$ of the form $g(\bm{x})=\Phi(\bm{x},1)$ for some map $\Phi:\mathbb{R}^d\times U\rightarrow\mathbb{R}^d$ such that
    \begin{itemize}
        \item $U \subset \R$ is an open interval containing $[0, 1]$,
        \item $\Phi(\bm{x},0)=\bm{x}$,
        \item $\Phi(\cdot,t)\in\DcRDr$ for any $t\in U$,
        \item $\Phi(\bm{x},s+t)=\Phi(\Phi(\bm{x},s),t)$ for any $s,t\in U$ with $s+t\in U$,
        \item $\Phi$ is $C^r$ on $\R^d \times U$,
        \item There exists a compact subset $K_\Phi \subset \R^d$ such that $\cup_{t \in U} \mathrm{supp}{\Phi(\cdot, t)} \subset K_\Phi$.
    \end{itemize}
\end{definition}
The difference between Definition~\ref{def: appendix flow endpoints} and the one in Corollary~2 of \citet{TeshimaCouplingbased2020} mainly lies in the last two conditions.
Technically, these two conditions are used in Section~\ref{sec:appendix:universality-proof-content} for showing that the partial derivative of $\Phi$ in $t$ at $t=0$ is Lipschitz continuous.

\begin{lemma}[Modified Corollary~2 of \citet{TeshimaCouplingbased2020}]
\label{lem: diffc2 is generated by flow endpoints}
Let $1 \leq r \leq \infty$ and $\FlowEnds{r}\subset\DcRDr$ be the set of all flow endpoints.
Then, the subset $H^r$ of $\DcRDr$ defined by
\[H^r:=\{ g_1\circ \cdots \circ g_n \ |\ n\ge1, g_1,\dots,g_n \in \FlowEnds{r}\}\]
forms a subgroup of $\DcRDr$ and it is a non-trivial normal subgroup.
\end{lemma}
\begin{proof}[Proof of Lemma~\ref{lem: diffc2 is generated by flow endpoints}]
First, we prove that $H^r$ forms a subgroup of $\DcRDr$.
By definition, for any $g, h \in H^r$, it holds that $g \circ h \in H^r$.
Also, $H^r$ is closed under inversion; to see this, it suffices to show that $\FlowEnds{r}$ is closed under inversion.
Let $g= \Phi(\cdot, 1) \in \FlowEnds{r}$. Consider the map $\phi:\mathbb{R}^d\times U\rightarrow\mathbb{R}^d$ defined by $\phi(\cdot, t) := \Phi^{-1}(\cdot, t)$.
It is easy to confirm that $\phi$ satisfies the conditions of Definition~\ref{def: appendix flow endpoints}, hence $g^{-1} = \phi(\cdot, 1)$ is an element of $\FlowEnds{r}$. Note that $\phi$ is confirmed to be $C^r$ on $\ReD \times U$ by applying the inverse function theorem to $(t, \x) \mapsto (t, \Phi(\x, t))$.

Next, we prove that $H^r$ is normal.
To show that the subgroup generated by $\FlowEnds{r}$ is normal, it suffices to show that $\FlowEnds{r}$ is closed under conjugation.
Take any $g\in \FlowEnds{r}$ and $h\in \DcRDr$, and let $\Phi$ be a flow associated with $g$.
Then, the function $\Phi': \mathbb{R}^d\times U \to \ReD$ defined by $\Phi'(\cdot, s) := h^{-1} \circ \Phi(\cdot, s) \circ h$ is a flow associated with $h^{-1}\circ g \circ h$ satisfying the conditions in Definition~\ref{def: appendix flow endpoints}, which implies $h^{-1}\circ g \circ h\in \FlowEnds{r}$, i.e., $\FlowEnds{r}$ is closed under conjugation.

Next, we prove that $H^r$ is non-trivial by constructing an element of $\FlowEnds{r}$ that is not the identity element.
First, consider the case $d = 1$.
Let $\tilde v: \Re \to \Re_{\geq 0}$ be a non-constant $C^\infty$-function such that $\supp{\tilde v} \subset [0, 1]$ and $\tilde v^{(k)}(0) = 0$ for any \(k \in \Na\).
Then define \(v : \Re \to \Re\) by
\[v(x) = \begin{cases}\tilde v(|x|)\frac{x}{|x|} & \text{ if } x \neq 0, \\ 0 & \text{ if } x = 0,\end{cases}\]
which is a \(C^\infty\)-function on \(\Re\) with a compact support.
Since $v$ is Lipschitz continuous and $C^\infty$, there exists $\IVPFunc{v}$ that is a $C^\infty$-function over $\Re \times \Re$; see Fact~\ref{fact:ODE solution exists for Lip} and \citep[Chapter~V, Corollary~4.1]{HartmanOrdinary2002}.
Let $K_v \subset \Re$ be a compact subset that contains $\supp{v}$. Then, by considering the ordinary differential equation by which $\IVPFunc{v}$ is defined, we see that $\bigcup_{t \in \Re} \supp\IVP{v}{\cdot}{t} \subset K_v$ and also that $\IVP{v}{x}{0} = x$.
We also have $\IVP{v}{x}{s+t} = \IVP{v}{\IVP{v}{x}{s}}{t}$ for any $s, t \in \Re$. In particular, we have $\IVP{v}{\cdot}{s}^{-1} = \IVP{v}{\cdot}{-s}$ for any $s \in \Re$. Therefore, we have $\IVP{v}{\cdot}{1} \in \FlowEnds{r}$. Since $v \not \equiv 0$, $\IVP{v}{\cdot}{1}$ is not an identity map and thus $\FlowEnds{r}$ is not trivial.
Next, we consider the case $d \geq 2$.
Take a $C^\infty$-function $\phi\colon \R\to \R$ with $\supp{\phi}= [1,2]$ and a nonzero skew-symmetric matrix $A$ (i.e. $A^\top=-A$) of size $d$, and
let $X(x):=\phi(\|x\|)A$.
We define a $C^\infty$-map $\Phi\colon \R^d\times \R\to \R^d$ by 
\[ \Phi(x,t):= \exp(t X(x))x. \]
Since $\exp( tX(x))$ is an orthogonal matrix for any $t\in \R$ and $x\in \R^d$, $\Phi$ is a $C^\infty$-flow on $\R^d$. 
Now, it is enough to show that there exists a compact set $K_\Phi\subset \R^d$ satisfying $\cup_{t\in \R}\supp{\Phi(\cdot, t)}\subset K_\Phi$. 
Let $K_\Phi:=\{x\in \R^d\ |\ \|x\|\leq 2\}$. 
Then the inclusion $\supp{\Phi(\cdot, t)}\subset K_{\Phi}$ holds for any $t\in\R$ since $X(x)=0$ for $x\in \R^d\setminus K_\Phi$.
\end{proof}

The following lemma allows us to approximate an autonomous ODE flow endpoint by approximating the differential equation. See Definition~\ref{def: autonomous ODE flow endpoints} for the definition of $\ODEFlowEnds{\cdot}$.
\begin{lemma}[Approximation of Autonomous-ODE flow endpoints]
\label{appendix:lem:ODE flow endpoint approximation}
Assume \(\NODEJacobiFuncClass \subset \Lipsp{}\) is a $\sup$-universal approximator for $\Lipsp{}$.
Then, \(\ODEFlowEnds{\NODEJacobiFuncClass{}}\) is a \(\sup\)-universal approximator for \(\ODEFlowEnds{\Lipsp{}}\).
\end{lemma}
\begin{proof}
Let $\phi \in \ODEFlowEnds{\Lipsp{}}$. Then, by definition, there exists $F \in \Lipsp{}$ such that $\phi = \IVP{F}{\cdot}{1}$.
Let $\LipConst{F}$ denote the Lipschitz constant of $F$.
In the following, we approximate $\IVP{F}{\cdot}{1}$ by approximating $F$ using an element of $\NODEJacobiFuncClass{}$.

Let $\varepsilon > 0$, and let $K \subset \R^d$ be a compact subset of $\R^d$.
We show that there exists $f \in \NODEJacobiFuncClass{}$ such that $\supKnorm{\IVP{F}{\cdot}{1} - \IVP{f}{\cdot}{1}} < \varepsilon$. Note that $\IVP{f}{\cdot}{\cdot}$ is well-defined because $\NODEJacobiFuncClass{} \subset \Lipsp{}$.
Define
\[
K' := \left\{\x \in \R^d \ \bigg|\ \inf_{\y \in \IVP{F}{K}{[0, 1]}} \|\x - \y\| \leq 2 e^{\LipConst{F}}\right\}.
\]
Then, $K'$ is compact. This follows from the compactness of $\IVP{F}{K}{[0, 1]}$: (i) $K'$ is bounded since $\IVP{F}{K}{[0, 1]}$ is bounded, and (ii) it is closed since the function $\min_{\y \in \IVP{F}{K}{[0, 1]}} \|\x - \y\|$ is continuous and hence $K'$ is the inverse image of a closed interval $[0, 2e^{\LipConst{F}}]$ by a continuous map.

Since $\NODEJacobiFuncClass{}$ is assumed to be a \(\sup\)-universal approximator for $\Lipsp{}$, for any $\delta > 0$, we can take $f \in \NODEJacobiFuncClass{}$ such that $\supRangenorm{K'}{f - F} < \delta$.
Let $\delta$ be such that $0 < \delta < \min\{\varepsilon / (2e^{\LipConst{F}}), 1\}$, and take such an $f$.

Fix $\x_0 \in K$ and define $\targetError{t} := \|\IVP{F}{\x_0}{t} - \IVP{f}{\x_0}{t}\|$.
Let $\Bound{} := \BoundDef{}$ and we show that
\[
\targetError{t} < 2\Bound{}
\]
holds for all $t \in [0, 1]$.
We prove this by contradiction. Suppose that there exists $t'$ for which the inequality does not hold. Then, the set $\mathcal{T} := \{t \in [0, 1] | \targetError{t} \geq 2 \Bound{}\}$ is not empty and
 thus $\tau := \inf \mathcal{T} \in [0, 1]$.
For this $\tau$, we show both $\targetError{\tau} \leq \Bound{}$ and $\targetError{\tau} \geq 2\Bound{}$.
First, we have
\begin{align*}
\targetError{\tau} &= \left\|\IVP{F}{\x_0}{\tau} - \IVP{f}{\x_0}{\tau}\right\| \\
&= \left\|\x_0 + \int_0^\tau F(\IVP{F}{\x_0}{t}) dt - \x_0 - \int_0^\tau f(\IVP{f}{\x_0}{t}) dt\right\| \\
&\leq \left\|\int_0^\tau (F(\IVP{F}{\x_0}{t}) - F(\IVP{f}{\x_0}{t})) dt\right\| \\
&\qquad + \left\|\int_0^\tau (F(\IVP{f}{\x_0}{t}) - f(\IVP{f}{\x_0}{t})) dt\right\|.
\end{align*}
The last term can be bounded as
\[
\left\|\int_0^\tau (F(\IVP{f}{\x_0}{t}) - f(\IVP{f}{\x_0}{t})) dt\right\| \leq \int_0^\tau \delta dt
\]
because of the following argument.
If $\tau = 0$, then both sides equal to zero, hence it holds with equality.
If $\tau > 0$, then for any $t < \tau$, we have $\IVP{f}{\x_0}{t} \in K'$ because $t < \tau$ implies $\targetError{t} \leq 2 \Bound{}$.
In this case, $\supRangenorm{K'}{F - f} < \delta$ implies the inequality.
Therefore, we have
\[
\targetError{\tau} \leq \LipConst{F}\int_0^\tau \targetError{t} dt + \int_0^\tau \delta dt.
\]
Now, by applying \Gronwall{}'s inequality \cite{GronwallNote1919}, we obtain
\[
\targetError{\tau} \leq \delta \tau e^{\LipConst{F} \tau} \leq \Bound{}.
\]
On the other hand, by the definition of $\mathcal{T}$ and the continuity of $\targetError{\cdot}$, we have $\targetError{\tau} \geq 2 \Bound{}$.
These two inequalities contradict.

Therefore, $\supKnorm{\IVP{F}{\cdot}{1} - \IVP{f}{\cdot}{1}} = \sup_{\x_0 \in K} \targetError{1} \leq 2 \Bound{} = 2\BoundDef{}$ holds.
Since $\delta < \varepsilon / (2e^{\LipConst{F}})$, the right-hand side is smaller than $\varepsilon$.
\end{proof}

Finally, we display a lemma that is useful in the case of $d = 1$. It is proved by convolving a smooth bump-like function.
\begin{fact}[Lemma~11 of \citet{TeshimaCouplingbased2020}]
\label{fact: d=1 smoothing}
Let $\tau: \Re \to \Re$ be a strictly increasing continuous function. Then, for any compact subset $K \subset \Re$ and any $\varepsilon > 0$, there exists a strictly increasing $C^\infty$-function $\tilde \tau$ such that \[\supKnorm{\tau - \tilde \tau} < \varepsilon.\]
\end{fact}

 \subsection{Proof of Theorem~\ref{thm: NODE is sup-universal}}
\label{sec:appendix:universality-proof-content}
\begin{proof}[Proof of Theorem~\ref{thm: NODE is sup-universal}]
Let $F \colon U \to \ReD$ be an element of $\CtwoDomainDiff$. 
Take any compact set $K\subset U$ and $\varepsilon>0$.
First, thanks to Fact~\ref{red to comp. supp. diff}, there exists a $G \in \DcRD$ and an affine transform $W \in \FLin$ such that \[\Restrict{W\circ G}{K}=\Restrict{F}{K}.\]
Now, if $d \geq 2$, then $2 \neq d + 1$, hence we can immediately use Fact~\ref{fact: simplicity} and Lemma~\ref{lem: diffc2 is generated by flow endpoints} to show that there exists a finite set of flow endpoints (Definition~\ref{def: appendix flow endpoints}) $g_1, \ldots, g_k \in \FlowEnds{2}$ such that
\[
G = g_k \circ \cdots \circ g_1.
\]
On the other hand, if $d = 1$, by Fact~\ref{fact: d=1 smoothing}, for any $\delta > 0$, we can find $\tilde G$ that is a $C^\infty$-diffeomorphism on $\Re$ such that $\supKnorm{G - \tilde G} < \delta$. Without loss of generality, we may assume that $\tilde G$ is compactly supported so that $\tilde G \in \DcRDCmd{\infty}$.
Then, we can use Fact~\ref{fact: simplicity} and Lemma~\ref{lem: diffc2 is generated by flow endpoints} to show that there exists a finite set of flow endpoints (Definition~\ref{def: appendix flow endpoints}) $g_1, \ldots, g_k \in \FlowEnds{\infty}$ such that
\[
\tilde G = g_k \circ \cdots \circ g_1.
\]

We now construct $f_j \in \Lipsp{}$ such that $g_j = \IVP{f_j}{\cdot}{1}$.
By Definition~\ref{def: appendix flow endpoints}, for each $g_j$ ($1 \leq j \leq k$), there exists an associated flow $\Phi_j$.
Now, define 
\[f_j(\cdot):=\left.\frac{\partial  \Phi_j(\cdot,t)}{\partial t}\right|_{t=0}.
\]
Then, $f_j \in \Lipsp{}$ because it is a compactly-supported $C^1$-map:
it is compactly supported since there exists a compact subset $K_j \subset \R^d$ containing the support of $\Phi(\cdot, t)$ for all $t$, and hence $\Phi(\cdot, t) - \Phi(\cdot, 0)$ is zero in the complement of $K_j$.

Now, $\Phi_j(\x, t) = \IVP{f_j}{\x}{t}$
since, by additivity of the flows,
\begin{align*}
    \frac{\partial \Phi_j}{\partial t}(\x,t)&=\lim_{s\rightarrow 0}\frac{\Phi_j(\x,t+s)-\Phi_j(\x,t)}{s}
    = \lim_{s\rightarrow 0}\frac{\Phi_j(\Phi_j(\x,t),s)-\Phi_j(\Phi_j(\x,t), 0)}{s}\\
    &= \left.\frac{\partial  \Phi_j(\Phi_j(\x,t),s)}{\partial s}\right|_{s=0}
    = f_j(\Phi_j(\x,t)),
\end{align*}
and hence it is a solution to the initial value problem that is unique.
As a result, we have $g_j = \Phi_j(\cdot, 1) = \IVP{f_j}{\cdot}{1}$.

By combining Fact~\ref{lemma:composition} and Lemma~\ref{appendix:lem:ODE flow endpoint approximation}, there exist $\phi_1, \ldots, \phi_k \in \Psi(\NODEJacobiFuncClass{})$ such that
\[
\supKnorm{g_k \circ \cdots \circ g_1 - \phi_k \circ \cdots \circ \phi_1} < \frac{\varepsilon}{\opnorm{W}},
\]
where $\opnorm{\cdot}$ denotes the operator norm.
Therefore, we have that $W \circ \phi_k \circ \cdots \circ \phi_1 \in \INNHNODE$ satisfies \begin{align*}
\supKnorm{F - W \circ \phi_k \circ \cdots \circ \phi_1}
&= \supKnorm{W \circ G - W \circ \phi_k \circ \cdots \circ \phi_1} \\
&\leq \opnorm{W}\supKnorm{g_k \circ \cdots \circ g_1 - \phi_k \circ \cdots \circ \phi_1} \\
&< \varepsilon
\end{align*}
if $d \geq 2$.
For $d = 1$, it can be shown that there exists $W \circ \phi_k \circ \cdots \circ \phi_1 \in \INNHNODE$ that satisfies $\supKnorm{F - W \circ \phi_k \circ \cdots \circ \phi_1} < \varepsilon$ in a similar manner.
\end{proof}

\section{Terminal time of autonomous-ODE flow endpoints}
\label{sec:appendix:remark-terminal-time}

In Definition~\ref{def: autonomous ODE flow endpoints}, the choice of the terminal value of the time variable, $t = 1$, is only technical.
To see this, let $T > 0$.
If we consider $w: \Re \to \ReD$ that is the solution of the initial value problem
$w(0) = \x, \dot{w}(t) = (T f)(w(t)) \ (t \in \Re)$
as well as $z: \Re \to \R^d$ that is the unique solution to
$z(0) = \x, \dot{z}(t) = f(z(t)) \ (t \in \Re)$,
then $w(t) = z(Tt)$ holds.
Therefore, $\IVP{f}{\x}{Tt} = \IVP{Tf}{\x}{t}$.

As a result, $\IVP{f}{\x}{T} = \IVP{Tf}{\x}{1}$ holds.
Therefore, it holds that
\[
\{\IVP{f}{\cdot}{T} \ |\ f\in \mathcal{F}\}
= \{\IVP{Tf}{\cdot}{1} \ |\ f\in \mathcal{F}\}
= \ODEFlowEnds{T\mathcal{F}}.
\]

Thus, even if we consider $T \neq 1$, if the set $\mathcal{F}$ is a cone, the set of the autonomous-ODE flow endpoints remains the same.

\section{Comparison between \(L^p\)-universality and \(\sup\)-universality}
In this section, we discuss the advantage of having a representation power guarantee in terms of the \(\sup\)-norm instead of the \(L^p\)-norm in function approximation tasks.

Roughly speaking, the function approximation should be robust under a slight change of norms, but $L^p$-universal approximation property can be sensitive to the choice of $p$.
To make this point, we construct an example: even if a model $g$ sufficiently approximates a target $f$ with the norm $\|\cdot\|_{1,K}$, the model $g$ may fail to approximate $f$ with $\|\cdot\|_{p,K}$ for any $p > 1$, even if $p$ is very close to $1$.

Let $h: (0,1) \rightarrow \Re$ be a strictly increasing function such that
\[
\begin{cases}
\|h\|_{p',[0,1]} &< \infty \text{ if } p'=1, \\
\|h\|_{p',[0,1]} &= \infty \text{ if } p'>1.
\end{cases}
\]
For example, $h(x)=-\sum_{k=1}^\infty x^{1/k-1}/k^3$ satisfies this condition.
Then, we define
\[g_n(x)=x+\frac{h(x)}{n}.\]

Now, the sequence $\{g_n\}_{n = 1}^\infty$ approximates $\Identity$ in $L^1$-norm in the sense that for any small $\varepsilon >0$, for sufficiently large $N$, it holds that
\begin{align}
    \|g_N - {\rm Id}\|_{1,[0,1]} &<\varepsilon, \label{siki for L1}\\
\end{align}
However, the same $g_N$ fails to approximate $\Identity$ in $L^p$-norm ($p > 1$) since it always holds that, for sufficiently small $\delta \in (0, 1/2)$,
\begin{align}
    \|g_N - {\rm Id}\|_{p,[\delta,1-\delta]} &\ge 1. \label{siki for Lp}
\end{align}
This example highlights that fixing $p$ first and guaranteeing approximation in $L^p$-norm may not suffice for guaranteeing the approximation in $L^{p'}$-norm ($p' > p$). On the other hand, having a guarantee in $\sup$-norm suffices for providing an approximation guarantee in $L^p$-norm for $p \ge 1$ simultaneously.
 
\end{appendices}
 
\end{document}